\documentclass{article}

\usepackage{fullpage}
\usepackage{amsmath,amsfonts,amssymb,amsthm}
\usepackage{hyperref}
\usepackage[capitalise]{cleveref}
\usepackage{natbib}
\usepackage{kpfonts}

\usepackage[utf8]{inputenc}
\usepackage{nicefrac}
\usepackage{xcolor}

\newcommand{\remove}[1]{}
\renewcommand{\cref}{\Cref}
\DeclareSymbolFont{AMSb}{U}{msb}{m}{n}
\DeclareMathSymbol{\N}{\mathbin}{AMSb}{"4E}
\DeclareMathSymbol{\Z}{\mathbin}{AMSb}{"5A}
\DeclareMathSymbol{\R}{\mathbin}{AMSb}{"52}
\DeclareMathSymbol{\Q}{\mathbin}{AMSb}{"51}
\DeclareMathSymbol{\erert}{\mathbin}{AMSb}{"50}
\DeclareMathSymbol{\I}{\mathbin}{AMSb}{"49}
\DeclareMathSymbol{\C}{\mathbin}{AMSb}{"43}

\newcommand{\Alg}{\AAA}
\newcommand{\AlgFindTukey}{\AAA_{\rm FindTukey}}
\newcommand{\AlgHalfSpace}{\AAA_{\rm LearnHalfSpace}}
\newcommand{\AlgRecConcave}{\AAA_{\rm RecConcave}}

\newcommand{\AAA}{\mathcal A}
\newcommand{\BBB}{\mathcal B}

\newcommand{\DDD}{\mathcal D}

\newcommand{\T}{\mathcal T}
\newcommand{\eps}{\varepsilon}

\newcommand{\error}{{\rm error}}
\newcommand{\db}{S}

\newcommand{\vspan}{\operatorname{\rm span}}
\newcommand{\poly}{\mathop{\rm poly}}

\newcommand{\hs}{\operatorname{\rm hs}}
\newcommand{\hp}{\operatorname{\rm hp}}
\newcommand{\td}{\operatorname{\rm td}}
\newcommand{\tdl}{\operatorname{\rm tdl}}
\newcommand{\halfspace}{\operatorname{\tt HALFSPACE}}

\newcommand{\set}[1]{\left\{ #1 \right\}}

\def\E{\operatorname*{\mathbb{E}}}
\def\1{\operatorname*{\mathbb{1}}}

\def\Q{\operatorname*{\mathbb{Q}}}
\def\poly{\mathop{\rm{poly}}\nolimits}

\newcommand{\pt}[1]{{\bf #1}}

\newtheorem{theorem}{Theorem}[section]
\newtheorem*{theorem*}{Theorem}
\newtheorem{lemma}[theorem]{Lemma}
\newtheorem*{lemma*}{Lemma}
\newtheorem{claim}[theorem]{Claim}
\newtheorem{proposition}[theorem]{Proposition}
\newtheorem{observation}[theorem]{Observation}

\theoremstyle{definition}
\newtheorem{definition}[theorem]{Definition}
\newtheorem{corollary}[theorem]{Corollary}

\newtheorem{fact}[theorem]{Fact}

\newcommand{\shay}[1]{\textcolor{blue}{\bf \{Shay: #1\}}}

\crefname{proposition}{proposition}{Propositions}

\title{Private Center Points and Learning of Halfspaces\thanks{A.~B.\ and K.~N.\ were supported by NSF grant no.~1565387. 
TWC: Large: Collaborative: Computing Over Distributed Sensitive Data. A.~B.\ was supported by ISF grant no.~152/17. Work done while A.~B.\ was visiting Georgetown University. U.~S.\ was supported by a gift from Google Ltd.}}
\author{
Amos Beimel\thanks{Dept.\ of Computer Science, Ben-Gurion University. {\tt amos.beimel@gmail.com}.}
\and
Shay Moran\thanks{Princeton University. {\tt shaymoran1@gmail.com}.}
\and
Kobbi Nissim\thanks{Dept. of Computer Science, Georgetown University. \tt{kobbi.nissim@georgetown.edu}.}
\and
Uri Stemmer\thanks{Dept.\ of Computer Science, Ben-Gurion University. {\tt u@uri.co.il}.}
}

\begin{document}

\maketitle
\begin{abstract}
We present a private learner for halfspaces over an arbitrary finite domain $X\subset \R^d$ with sample complexity $\poly(d,2^{\log^*|X|})$. The building block for this learner is a differentially private algorithm for locating an approximate center point of $m>\poly(d,2^{\log^*|X|})$ points -- a high dimensional generalization of the median function. Our construction establishes a relationship between these two problems that is reminiscent of the relation between the median and learning one-dimensional thresholds [Bun et al.\ FOCS '15]. This relationship suggests that the problem of privately locating a center point may have further applications in the design of differentially private algorithms.

We also provide a lower bound on the sample complexity for privately finding a point in the convex hull. For approximate  differential privacy, we show a lower bound of $m=\Omega(d+\log^*|X|)$, whereas for pure differential privacy $m=\Omega(d\log|X|)$.

\end{abstract}

\section{Introduction}

Machine learning models are often trained on sensitive personal information, e.g.,\ when analyzing healthcare records or social media data. There is hence an increasing awareness and demand for privacy preserving machine learning technology.
This motivated the line of works on {\em private learning}, initiated by \cite{KLNRS11}, {which} provides strong (mathematically proven) privacy protections for the training data. Specifically, these works aim at achieving {\em differential privacy}, a strong notion of privacy that is now increasingly being adopted by both academic researchers and industrial companies. Intuitively, a private learner is a PAC learner that guarantees that every single example has almost no effect on the resulting classifier. Formally, a private learner is a PAC learner that satisfies differential privacy w.r.t.\ its training data. The definition of differential privacy is,

\begin{definition}[\cite{DMNS06}]\label{def:dpIntro}
Let $\AAA$ be a randomized algorithm that operates on databases.
Algorithm $\AAA$ is $(\eps,\delta)$-{\em differentially private} if for any two databases $S,S'$ that differ on one row, and any event~$T$, we have 
$\Pr[\AAA(S)\in T]\leq e^{\eps}\cdot \Pr[\AAA(S')\in T]+\delta.$ 
The {notion} is referred to as {\em pure} differential privacy when $\delta=0$, and {\em approximate} differential privacy when $\delta>0$.
\end{definition}

The initial work of \cite{KLNRS11} 
 showed that any concept class $C$ is privately learnable with sample complexity $O(\log |C|)$ (we omit in the introduction the dependencies on accuracy and privacy parameters). Non-privately, $\Theta(VC(C))$ samples are necessary and sufficient to PAC learn $C$, and much research has been devoted to understanding how large the gap is between the sample complexity of private and non-private PAC learners. 
For {\em pure} differential privacy, it is known that a sample complexity of $\Theta(\log|C|)$ is required even for learning some simple concept classes such as one-dimensional thresholds, axis-aligned rectangles, balls, and halfspaces \citep{BBKN12,BNS13,FX14}. That is, generally speaking, learning with pure differential privacy requires sample complexity proportional to log the size of the hypothesis class. For example, in order to learn halfspaces in $\R^d$, one must consider some {\em finite} discretization of the problem, e.g.\ by assuming that input examples come from a finite set $X\subseteq\R^d$. A halfspace over $X$ is represented using $d$ point from $X$, and hence, learning halfspaces over $X$ with pure differential privacy requires sample complexity $\Theta(\log{|X| \choose d})=O(d\log|X|)$. In contrast, learning halfspaces non-privately requires sample complexity $O(d)$. In particular, when the dimension $d$ is constant, learning halfspaces non-privately is achieved with {\em constant} sample complexity, while learning with pure differential privacy requires sample complexity that is proportional to the representation length of domain elements.

For {\em approximate} differential privacy, the current understanding is more limited. 
Recent results established that the class of one-dimensional thresholds over a domain $X\subseteq\R$ requires sample complexity between $\Omega(\log^*|X|)$ and $2^{O(\log^*|X|)}$ (\cite{BNS13b,BNSV15,Bun16,ALMM18}). 
On the one hand, these results establish a separation between what can be learned with or without privacy, as they imply that privately learning one-dimensional thresholds over an infinite domain is impossible. 
On the other hand, these results show that, unlike with pure differential privacy, the sample complexity of learning one-dimensional thresholds can be much smaller than $\log|C|=\log|X|$. \cite{BNS13b} also established an upper bound of $\poly(d\cdot2^{\log^*|X|})$ for privately learning the class of axis-aligned rectangles over $X\subseteq\R^d$. 
In a nutshell, this concludes our current understanding of the sample complexity of approximate private learning. In particular, before this work, it was not known whether similar upper bounds (that grow slower than $\log|C|$) can be established for ``richer'' concept classes, such as halfspaces, balls, and polynomials.


We answer this question positively, focusing on privately learning halfspaces. The class of halfspaces forms an important primitive in machine learning as learning halfspaces implies learning many other concept classes~(\cite{BL98}). In particular, it is the basis of  popular algorithms such as neural nets and kernel machines, as well as various geometric classes (e.g., polynomial threshold functions, polytopes, and $d$-dimensional balls). 

\subsection{Our Results}

Our approach for privately learning halfspaces is based on a reduction to the task of privately finding a point in the convex hull of a given input dataset. That is, towards privately learning halfspaces we first design a sample-efficient differentially private algorithm for identifying a point in the convex hull of the given data, and then we show how to use such an algorithm for privately learning halfspaces.

\paragraph{Privately finding a point in the convex hull.}
We initiate the study of privately finding a point in the convex hull of a dataset $S\subseteq X\subseteq\R^d$. Even though this is a very natural problem (with important applications, in particular to learning halfspaces), it has not been considered before in the literature of differential privacy. One might try to solve this problem using the {\em exponential mechanism} of \cite{MT07}, which, given a dataset and a {\em quality function}, privately identifies a point with approximately maximum quality. To that end, one must first settle on a suitable quality function such that if a point $\pt{x}\in X$ has a high quality then this point is guaranteed to be in the convex hull of $S$. Note that the indicator function $q(\pt{x})=1$ if and only if $\pt{x}$ is in the convex hull of $S$ is not a good option, as every point $\pt{x}\in X$ has quality either 0 or 1, and the exponential mechanism only guarantees a solution with {\em approximately} maximum quality (with additive error larger than 1).

Our approach is based on the concept of {\em Tukey-depth} \citep{Tuk75}. Given a dataset $S\subseteq\R^d$,  a point $\pt{x}\in\R^d$ has the Tukey-depth at most $\ell$ if there exists a set $A \subseteq S$ of size $\ell$ such that $\pt{x}$ is not in the convex hull of $S\setminus A$. See \cref{def:Tukey} for an equivalent definition that has a geometric flavor. Instantiating the exponential mechanism with the Tukey-depth as the quality function results in a private algorithm for identifying a point in the convex hull of a dataset $S\subseteq X\subseteq\R^d$ with sample complexity $\poly(d,\log|X|)$.\footnote{We remark that the domain $X$ does not necessarily contain a point with a high Tukey-depth, even when the input points come from $X\subseteq\R^d$. Hence, one must first {\em extend} the domain $X$ to make sure that a good solution exists. This results in a private algorithm with sample complexity $O(d^3\log|X|)$.} 
We show that this upper bound can be improved to $\poly(d,2^{\log^*|X|})$. Our construction utilizes an algorithm by~\cite{BNS13b} for approximately maximizing (one-dimensional) quasi-concave functions with differential privacy (see Definition~\ref{def:quasiConcave} for quasi-concavity). To that end, we show that it is possible to {find a point with high}  Tukey-depth in iterations over the axes, and show that the {appropriate functions are} indeed quasi-concave. This allows us to instantiate the algorithm of \cite{BNS13b} to identify a point in the convex hull of the dataset one coordinate at a time. We obtain the following theorem.

\begin{theorem}[Informal]
Let $ \eps \leq 1$ and $\delta < 1/2$ and let $X \subset\R^d$. There exists an $(\eps,\delta)$-differentially private algorithm that given a dataset $S\in X^m$ identifies (w.h.p.) a point in the convex hull of $S$, provided that   
$m=|S|=\poly\left(d,2^{\log^*|X|},\frac{1}{\eps},\log\frac{1}{\delta}\right)$.
\end{theorem}

In fact, our algorithm returns a point with a large Tukey-depth, which is in particular a point in the convex hull of the dataset. This fact will be utilized by our reduction from learning halfspaces, and will allow us to get improved sample complexity bounds on privately learning halfspaces.

\paragraph{A privacy preserving reduction from halfspaces to convex hull.}
Our reduction can be thought of as a generalization of the results by \cite{BNSV15}, who showed that the task of privately learning {\em one-dimensional} thresholds is equivalent to the task of privately solving the {\em interior point problem}. In this problem, given a set of input numbers, the task is to identify a number between the minimal and the maximal input numbers. Indeed, this is exactly the one dimensional version of the convex-hull problem we consider. However, the reduction of \cite{BNSV15} does not apply for halfspaces, and we needed to design a different reduction.

Our reduction is based on the sample and aggregate paradigm: 
assume a differentially private algorithm~$\AAA$ which gets a (sufficiently large) dataset $D\subset \R^d$ and returns a point in the convex hull of $D$. 
This can be used to privately learn halfspaces as follows. 
Given an input sample $S$, partition it to sufficiently many subsamples $S_1,\dots,S_k$, 
and pick for each $S_i$ an arbitrary halfspace $h_i$ which is consistent with $S_i$. Next, apply $\AAA$ to privately find a point in the convex hull of the $h_i$'s (to this end represent each $h_i$ as a point in~$\R^{d+1}$ via its normal vector and bias), and output the halfspace $h$ corresponding to the returned point. It can be shown that if each of the $h_i$'s has a sufficiently low generalization error, which is true if the sample is big enough, then the resulting (privately computed) halfspace also has a low generalization error. 
Instantiating this reduction with our algorithm for the convex hull we get the following theorem.





\begin{theorem}[Informal]\label{thm:introHalfspaces}
Let $ \eps \leq 1$ and $\delta < 1/2$ and let $X \subset\R^d$. There exists an $(\eps,\delta)$-differentially private $(\alpha,\beta)$-PAC learner for halfspaces over examples from $X$ with sample complexity 
$m=\poly\left(d,2^{\log^*|X|},\frac{1}{\alpha\eps},\log\frac{1}{\beta\delta}\right)$.
\end{theorem}

In particular, for any constant $d$, \cref{thm:introHalfspaces} gives a private learner for halfspaces over $X\subseteq\R^d$ with sample complexity $2^{O(\log^*|X|)}$. Before our work, this was known only for $d=1$.

\paragraph{A lower bound for finding a point in the convex hull.}
Without privacy considerations, finding a point in the convex hull of the data is trivial. Nevertheless, we show that any $(\eps,\delta)$-differentially private algorithm for this task (in $d$ dimensions) must have sample complexity $m=\Omega(\frac{d}{\eps}\log\frac{1}{\delta}+{\log^* {\lvert X\rvert}})$. {In comparison, our algorithm requires sample of size at least
$\tilde{O}(d^{2.5}2^{O(\log^*{\lvert X\rvert})}/\epsilon)$ (ignoring the dependency on $\delta$ and $\beta$).}

Recall that the sample complexity of privately learning a class $C$ is always at most $O(\log|C|)$. Hence, it might be tempting to guess that a sample complexity of $m=O(\log|X|)$ should suffice for privately finding a point in the convex hull of a dataset $S\subseteq X\subseteq\R^d$, even with pure $(\eps,0)$-differential privacy. We show that this is not the case, and that any pure $(\eps,0)$-differentially private algorithm for this task must have sample complexity $m=\Omega(\frac{d}{\eps}\log|X|)$.

\subsection{Other Related Work}\label{sec:otherRelated}
Most related to our work is the work on private learning and its sample and time complexity by \cite{BDMN05, KLNRS11,BDMN05,BBKN12,CH11,BNS13,FX14,BNS13b,BNSV15,BZ16}. As some of these works demonstrate efficiency gaps between private and non-private learning, alternative models have been explored including semi-supervised learning (\cite{BNS15}), learning multiple concepts (\cite{BNS16}), and prediction (\cite{DworkF18}, \cite{BassilyTT18}).

\cite{DV08} showed an efficient (non-private) learner for halfspaces that works in (a variant of) the {\em statistical query (SQ)} model of \cite{Kearns98}. It is known that SQ learners can be transformed to preserve differential privacy \citep{BDMN05}, and the algorithm of \cite{DV08} yields a differentially private efficient learner for halfspaces over examples from $X\subseteq\R^d$ with sample complexity $\poly(d,\log|X|)$. 
Another related work is that of \cite{HsuRRU14} who constructed an algorithm for {\em approximately} solving linear programs with differential privacy. While learning halfspaces {\em non-privately} easily reduces to solving linear programs, it is not clear whether the results of \cite{HsuRRU14} imply a {\em private} learner for halfspaces (due to the types of errors they incur).

\section{Preliminaries}

In this section we introduce a tool that enables our constructions, describe the geometric object we use throughout the paper, and present some of their properties. 

\paragraph{Notations.} 
The input of our algorithm is a multiset $S$ whose elements are taken (possibly with repetition) from a set $X$. We will abuse notation and write that $S \subseteq X$. Databases $S_1$ and $S_2$ are called {\em neighboring} if they differ in exactly one entry. Throughout this paper we use $\eps$ and $\delta$ for the privacy parameters, $\alpha$ for the error parameter, and $\beta$ for the confidence parameter, and $m$ for the sample size.
In this appendix we define differentially private algorithms and the PAC learning model.

\subsection{Preliminaries from Differential Privacy} 

Consider a database where each record contains information of an individual. An algorithm is said to preserve differential privacy if a change of a single record of the database (i.e., information of an individual) does not significantly change the output distribution of the algorithm. Intuitively, this means that the information infer about an individual from the output of a differentially-private algorithm is similar to the information that would be inferred had the individual's record been arbitrarily modified or removed. Formally:

\begin{definition}[Differential privacy~\citep{DMNS06,DKMMN06}] \label{def:dp} 
A randomized algorithm $\Alg$ is $(\eps,\delta)$-differentially private if for all neighboring databases $\db_1,\db_2\in X^m$, and for all sets $\mathcal{F}$ of outputs,
\begin{eqnarray}
\label{eqn:diffPrivDef}
  & \Pr[\Alg(\db_1) \in \mathcal{F}] \leq \exp(\eps) \cdot \Pr[\Alg(\db_2) \in \mathcal{F}] + \delta,  &
\end{eqnarray}
where the probability is taken over the random coins of $\Alg$. 
When $\delta=0$ we omit it and say that $\Alg$ preserves $\eps$-differential privacy.
\end{definition}
We use the term {\em pure} differential privacy when $\delta=0$ and the term {\em approximate} differential privacy when $\delta>0$, in which case $\delta$ is typically a negligible function of the database size $m$.

We will later present algorithms that access their input database using (several) differentially private algorithms. We will use the following composition theorems. 

\begin{theorem}[Basic composition]\label{thm:composition1}
If $\Alg_1$ and $\Alg_2$ satisfy $(\eps_1,\delta_1)$ and $(\eps_2,\delta_2)$ differential privacy, respectively, then their concatenation $\Alg(S)=\langle \Alg_1(S),\Alg_2(S) \rangle$ satisfies $(\eps_1+\eps_2,\delta_1+\delta_2)$-differential privacy.
\end{theorem}

Moreover, a similar theorem holds for the adaptive case, where an algorithm  uses  $k$ {\em adaptively chosen} differentially private algorithms (that is, when the choice of the next differentially private algorithm that is used depends on the outputs of the previous differentially private algorithms).

\begin{theorem}[\citep{DKMMN06, DworkLei}]\label{thm:composition3}
An algorithm that adaptively uses $k$  algorithms that preserves $(\eps/k,\delta/k)$-differential privacy (and does not access the database otherwise) ensures $(\eps,\delta)$-differential privacy.
\end{theorem}

Note that the privacy guaranties of the above bound deteriorates linearly with the number of interactions. By bounding the {\em expected} privacy loss in each interaction (as opposed to worst-case), \cite{DRV10} showed the following stronger composition theorem, where privacy deteriorates (roughly) as $\sqrt{k}\eps+k\eps^2$ (rather than $k\eps$).

\begin{theorem}[Advanced composition~\cite{DRV10}, restated]\label{thm:composition2}
Let $0<\eps_0,\delta'\leq1$, and let $\delta_0\in[0,1]$. An algorithm that adaptively uses $k$ algorithms that preserves $(\eps_0,\delta_0)$-differential privacy (and does not access the database otherwise) ensures $(\eps,\delta)$-differential privacy, where $\eps=\sqrt{2k\ln(1/\delta')}\cdot\eps_0+2k\eps_0^2$ and $\delta = k\delta_0+\delta'$.
\end{theorem}

\subsection{Preliminaries from Learning Theory}

We next define the probably approximately correct (PAC) model of~\cite{Valiant84}.
A concept $c:X\rightarrow \{0,1\}$ is a predicate that labels {\em examples} taken from the domain $X$ by either 0 or 1.  A \emph{concept class} $C$ over $X$ is a set of concepts (predicates) mapping $X$ to $\{0,1\}$. A learning algorithm is given examples sampled according to an unknown probability distribution $\DDD$ over $X$, and labeled according to an unknown {\em target} concept $c\in C$. The learning algorithm is successful when it outputs a hypothesis $h$ that approximates the target concept over samples from $\DDD$. More formally:

\begin{definition}
The {\em generalization error} of a hypothesis $h:X\rightarrow\{0,1\}$ is defined as 
$$\error_{\DDD}(c,h)=\Pr_{x \sim \DDD}[h(x)\neq c(x)].$$ 
If $\error_{\DDD}(c,h)\leq\alpha$ we say that $h$ is {\em $\alpha$-good} for $c$ and $\DDD$.
\end{definition}

\begin{definition}[PAC Learning~\citep{Valiant84}]\label{def:PAC}
Algorithm $\Alg$ is an {\em $(\alpha,\beta,m)$-PAC learner} for a concept
class $C$ over $X$ using hypothesis class $H$ if for all 
concepts $c \in C$, all distributions $\DDD$ on $X$,
given an input of $m$ samples $\db =(z_1,\ldots,z_m)$, where $z_i=(x_i,c(x_i))$ and each $x_i$
is drawn i.i.d.\ from $\DDD$, algorithm $\Alg$ outputs a
hypothesis $h\in H$ satisfying
$$\Pr[\error_{\DDD}(c,h)  \leq \alpha] \geq 1-\beta,$$
where the probability is taken over the random choice of
the examples in $\db$ according to $\DDD$ and the random coins  of the learner $\Alg$.
If $H\subseteq C$ then $\Alg$ is called a {\em proper} PAC learner; otherwise, it is called an {\em improper} PAC learner.
\end{definition}

\begin{definition}
For a labeled sample $\db=(x_i,y_i)_{i=1}^m$, the {\em empirical error} of $h$ is
$$\error_S(h) = \frac{1}{m} |\{i : h(x_i) \neq y_i\}|.$$
\end{definition}

\subsection{Private Learning}\label{sec:PPAC}
Consider a learning algorithm $\AAA$ in the probably approximately correct (PAC) model of~\cite{Valiant84}. We say that $\AAA$ is a {\em private} learner if it also satisfies differential privacy w.r.t.\ its training data. Formally,
\begin{definition}[Private PAC Learning~\citep{KLNRS11}]
Let $\Alg$ be an algorithm that gets an input $\db =(z_1,\ldots,z_m)${, where each $z_i$ is a labeled example}. Algorithm $\Alg$ is an {\em $(\eps,\delta)$-differentially private $(\alpha,\beta)$-PAC learner with sample complexity $m$} for a concept
class $C$ over $X$ using hypothesis class $H$ if
\begin{description}
\item{\sc Privacy.} Algorithm $\Alg$ is $(\eps,\delta)$-differentially private (as in  \cref{def:dp});
\item{\sc Utility.} Algorithm $\Alg$ is an {\em $(\alpha,\beta)$-PAC learner} for $C$ with sample complexity $m$ using hypothesis class $H$ {(as in  \cref{def:PAC})}.
\end{description}
\end{definition}

Note that the utility requirement in the above definition is an average-case requirement, as the learner is only required to do well on typical samples (i.e., samples drawn i.i.d. from a distribution $\DDD$ and correctly labeled by a target concept $c\in C$). In contrast, the privacy requirement is a worst-case requirement, {which} must hold for every pair of neighboring databases (no matter how they were generated, even if they are not consistent with any concept in $C$).

\subsection{A Private Algorithm for Optimizing Quasi-concave Functions -- $\AlgRecConcave$}
We next describe properties of an algorithm $\AlgRecConcave$ of \cite{BNS16a}. This algorithm
is given a quasi-concave function $Q$ (defined below) and privately finds a point $x$ such that $Q(x)$ is close to its maximum provided that the maximum of $Q(x)$ is large enough
(see~(\ref{eq:largeQ})). 
   
\begin{definition}\label{def:quasiConcave} 
A function $Q(\cdot)$ is quasi-concave if $Q(\ell) \geq \min\set{Q(i),Q(j)}$ for every $i < \ell  < j$.
\end{definition}

\begin{definition}[Sensitivity] The sensitivity of a function $f : X^m \rightarrow \R$ is the smallest $k$ such that for
every neighboring $D,D' \in X^m$, we have $|f(D)-f(D')
| \leq k$.
\end{definition}


\begin{proposition}[Properties of Algorithm $\AlgRecConcave$~\citep{BNS16a}]\label{prop:aRecConcave}
Let $Q:X^*\times\tilde{X}\rightarrow\R$ be a sensitivity-1 function (that is, for every $x \in \tilde{X}$, the function
$Q(\cdot,x)$ has sensitivity $1$).
Denote $\tilde{T}=|\tilde{X}|$ and let $\alpha\leq\frac{1}{2}$ and  $\beta,\eps,\delta,r$ be parameters.
There exits an $(\eps,\delta)$-differentially private algorithm, called $\AlgRecConcave$, such that the following holds.
If $\AlgRecConcave$ is executed on a database $S\in X^*$ such that $Q(S,\cdot)$ is quasi-concave and in addition
\begin{equation}
\label{eq:largeQ}
\max_{i\in\tilde{X}}\{Q(S,i)\} \geq r \geq
8^{\log^* \tilde{T}} \cdot \frac{12 \log^* \tilde{T}}{\alpha\eps}\log\Big(\frac{192(\log^* \tilde{T})^2}{\beta\delta}\Big),
\end{equation}
then with probability at least $1-\beta$ the algorithm outputs an index $j$ s.t.\ 
$Q(S,j)\geq(1-\alpha)r$.
\end{proposition}

\remove{
We next present an alternative definition of quasi-concave functions that is more convenient to use.
\begin{observation}
Let $i_{\rm max}$ be a value such that $Q(i_{\rm max})$ is maximal.
A function $Q(Â·)$ is quasi-concave if and only if $Q(i) \leq Q(j)$ for all $i,j\in \tilde X$ such that $i < j < i_{\rm max}$ or $i_{\rm max} < j < i$.

\begin{enumerate}
\item
The function is non-decreasing  before $i_{\rm max}$, that is 
$Q(i) \leq Q(j)$ for every $i,j\in \tilde{X}$ such that $i < j < i_{\rm max}$, {\em and}
\item
The function is non-increasing after $i_{\rm max}$, that is 
$Q(i) \geq Q(j)$ for every $i,j\in \tilde{X}$ such that $i_{\rm max} < i < j$.
\end{enumerate}
\end{observation} 
}

\begin{claim}\label{c:minquasiconcave}
Let $\{f_t\}_{t\in\mathcal{T}}$ be a finite family of quasi-concave functions. Then, $f(x)=\min_{t\in \mathcal{T}}f_t(x)$ is also quasi-concave.
\end{claim}
\begin{proof}
Let $i \leq \ell  \leq j$. Then,
\begin{align*}
f(\ell) &=\min_{t\in \mathcal{T}}f_t(\ell)\\
        &\geq \min_{t\in\mathcal{T}}\{\min\{f_t(i),f_t(j)\}\} \tag{$f_t$ is quasi-concave, $\forall t\in\mathcal{T}$}\\
        &= \min\{\min_{t\in \mathcal{T}}f_t(i),\min_{t\in \mathcal{T}} f_t(j)\}\\
        &=\min\{f(i),f(j)\}.
\end{align*}
\end{proof}

\subsection{Halfspaces, Convex Hull, and Tukey Depth}\label{def:Tukey}

We next define the geometric objects we use in this paper.



\begin{definition}[Halfspaces and Hyperplanes]
Let $X \subset \R^d$. 
For $a_1,\dots,a_d,w\in \R$, let the halfspace
$\hs_{a_1,\dots,a_d,w}:X \rightarrow\{0,1\}$
be defined as $\hs_{a_1,\dots,a_d,w}(x_1,\dots,x_d)=1$ if and only if  $\sum_{i=1}^d a_i x_i \geq w$. Define the concept class $\halfspace(X) = \{\hs_{a_1,\dots,a_d,w}\}_{a_1,\dots,a_d,w \in \R}$.
We say that a halfspace $\hs$ contains a point $\pt{x} \in \R^d$ if $\hs(\pt{x})=1$.
The hyperplane $\hp_{a_1,\dots,a_d,w}$ defined by $a_1,\dots,a_d,w$ is the set of all points $\pt{x}=(x_1,\dots,x_d)$ such that $\sum_{i=1}^d a_i x_i = w$.
\end{definition}

%
\begin{definition}
Let $S\subset \R^d$ be a finite multiset of points. A point $\pt{x} \in \R^d$ is in the convex hull of $S$ if $\pt{x}$ is a convex combination of the elements of $S$, that is, there exists non-negative
numbers $\set{a_{\pt{y}}}_{\pt{y} \in S}$ such that
$\sum_{\pt{y} \in S} a_\pt{y} =1$ and
$\sum_{\pt{y} \in S} a_\pt{y} \pt{y}=\pt{x}$.
\end{definition}

We next define the Tukey median of a point, which is a generalization of a median to $\R^d$.
\begin{definition}[Tukey depth~\citep{Tuk75}]
Let $S\subset \R^d$ be a finite {multiset} 
of points. The Tukey depth of a point $\pt{x} \in \R^d$ with respect to $S$, denoted by $\td(\pt{x})$, is the minimum number of points in $S$ contained in a halfspace containing  the point $x$,
that is, 
$$\td(\pt{x})=\min_{\hs \in \halfspace_{d,T} ,\hs(\pt{x})=1}|\set{\pt{y}\in S:\hs(\pt{y})=1}|.$$
The Tukey median of $S$ is a point maximizing the Tukey depth. A centerpoint is a point of depth at least~$|S|/(d + 1)$. 
\end{definition}

\begin{observation}
\label{obs:Tukey}
The Tukey depth of a point is a sensitivity one function of the 
multiset $S$.
\end{observation}
\remove{
\knote{this definition is not clear to me} We say that the direction $(a_1,\ldots,a_d)$
is in {\it general position} with respect to $S$
if the numbers $\sum_i a_iy_i$ are distinct
when $(a_1\ldots a_d)$ ranges over all points in $S$. \knote{should this be ``$(y_1\ldots y_d)$ ranges over all points in $S$''??}
\shay{The definition says that no two vectors in $S$ have the same inner product with $(a_1,\ldots, a_d)$. Would it be clearer if stated this way? 
I needed this definition in order to fix the sensitivity argument by Amos, but Uri wanted to use an alternative argument which does not require this definition and then we can remove it.}
Note that since $S$ is finite, a standard perturbation argument implies that\knote{this argument holds because the set $S$ is finite, right? if correct, let's mention this.}\shay{added.}
every halfspace $\hs$ can be written as
$\hs_{a_1,\ldots,a_d,w}$ where the direction
$(a_1,\ldots,a_d)$ is in general position with respect to $S$.
Let $\mathsf{GP}(S)\subseteq\mathbb{R}^d$ denote the set of all directions
that are in general position with respect to $S$.
}

\begin{claim}[Tukey depth, alternative definition]\label{c:tukeyalt}
Let $S\subset \R^d$ be a multiset of points. 
For a given $a_1,\dots,a_d\in \R$ define the function
\begin{equation}
\label{eq:tdef}
t_{a_1,\dots,a_d}(w)\triangleq\min\set{\left|\{(y_1,\dots,y_d)\in S: \sum_{i=1}^d a_i y_i \geq w\}\right|,\left|\{(y_1,\dots,y_d)\in S: \sum_{i=1}^d a_i y_i \leq w\}\right|}.
\end{equation}
Then, 
\begin{equation}
\label{eq:td}
\td(x_1,\dots,x_n)=\min_{(a_1,\dots,a_d)\in\R}t_{a_1,\dots,a_d}\left(\sum_{i=1}^d a_i x_i \right).
\end{equation}
\end{claim}

\begin{claim}[Tukey depth, another alternative definition]\label{c:tukeyalt2}
Let $S\subset \R^d$ be a multiset of points. 
The Tukey-depth of a point $\pt{x}$ is the size of the smallest set $A \subseteq S$ such that $\pt{x}$ is not in the convex-hull of $S \setminus A$. 
\end{claim}

\begin{claim}[\cite{Yaglom61book,Edelsbrunner87book}]\label{c:centerpoint}
Let $S\subset \R^d$ be a multiset of points. 
There exists $\pt{x} \in \R^d$ such that $\td(\pt{x}) \geq |S|/(d + 1)$. 
\end{claim}

Thus, a centerpoint always exists and a Tukey median must be a centerpoint. However, not every centerpoint is a Tukey median. 
We will use the following regarding the set of points of all points whose Tukey depth is at least $r$.
\begin{fact}[see e.g.~\cite{Xiaohui14Tukey}]\label{fact:counting}
Let $S\subseteq \R^d$ be a multiset of points and $r > 0$. 
Define $\T(r) = \{\pt{x}\in\mathbb{R}^d : \td(\pt{x}) \geq r\}$.
Then $\T(r)$ is a polytope whose faces are supported by affine subspaces 
that are spanned by points from~$S$.
\end{fact}
So, for example the set of all Tukey medians is a polytope and if it is $d$-dimensional then each of its facet is
supported by a hyperplane that passes through $d+1$ points from $S$.

\section{Finding a Point in the Convex Hull}
\label{sec:FindingPoint}

Our goal is to privately find a point in the convex hull of a set of input points (i.e., the database). We will actually achieve a stronger task and find a point whose Tukey depth is at least $|S|/2(d+1)$ (provided that $|S|$ is large enough). Observe that $\pt{x}$ is in the convex hull of $S$ if and only if $\td(\pt{x}) >0$. As we mentioned in the introduction, finding a point whose Tukey depth is high results in a better learning algorithms for halfspaces.

The idea of our algorithm is to find the point $\pt{x}=(x_1,\dots,x_d)$ coordinate after coordinate: we
use $\AlgRecConcave$ to find a value $x^*_1$ that can be extended by some $x_2,\ldots,x_d$ so that the depth of $(x_1^*,x_2\ldots,x_d)$ is close to the depth of the Tukey median, then we find a value $x_2^*$ so that there is a point $(x^*_1,x_2^*,x_3\dots,x_d)$ whose depth is close to the depth of the Tukey median, and so forth until we find all coordinates.  The parameters in $\AlgRecConcave$ are set such that in each step we lose depth of at most $n/2(d+1)^2$ compared to the Tukey median, resulting in a point  $(x^*_1,\dots,x^*_d)$ whose depth is at most $n/2(d+1)$ less than the depth of the Tukey median, i.e., its depth is at least $n/2(d+1)$.    

\subsection{Defining a Quasi-Concave Function}

To apply the above approach, we need to prove that the functions considered in the algorithm $\AlgRecConcave$
are quasi-concave.
\begin{definition}
For every $1 \leq i \leq d$ and every $x^*_1,\dots,x^*_{i-1} \in \R$, define
$$Q_{x^*_1,\dots,x^*_{i-1}}(x_i)\triangleq \max_{x_{i+1},\dots,x_d \in \R} \td (x^*_1,\dots,x^*_{i-1}, x_i,\dots,x_d).$$
\end{definition}



We next prove that $Q_{x^*_1,\dots,x^*_{i-1}}(x_i)$ is quasi-concave. Towards this goal, 
we first prove that the function~$t_{a_1,\dots,a_d}(w)$, defined in \cref{eq:tdef}, is quasi-concave.

\begin{claim}
\label{cl:t-concave}
For every  $a_1,\dots,a_{d} \in \R$,  the function
$t_{a_1,\dots,a_{d}}(w)$ is quasi-concave.
\end{claim}
\begin{proof}
Define
$
f_1(w) = \bigl\lvert\{(y_1,\dots,y_d)\in S: \sum_{i=1}^d a_i y_i \geq w\}\bigr\rvert,$ and 
$f_2(w) = \bigl\lvert\{(y_1,\dots,y_d)\in S: \sum_{i=1}^d a_i y_i \leq w\}\bigr\rvert$. Note that $t_{a_1,\dots,a_{d}}(w) = \min\{{f_1(w),f_2(w)}\}$.
These functions count the number of points in $S$ on the two (closed) sides of the hyperplane $\hp_{a_1\ldots a_d,w}$.
The claim follows by Claim~\ref{c:minquasiconcave} since both $f_1,f_2$ are quasi-concave
(in fact, both are monotone).
\end{proof}

We next prove that the restriction of the Tukey depth function to a line is quasi-concave.
This lemma is implied by Fact~\ref{fact:counting} (implying that the set of points whose Tukey depth is at least $r$ is convex). For completeness, we supply a full proof of the claim.

\begin{claim}
\label{cl:tdl-concave}
Fix $\alpha_1,\dots,\alpha_d,\beta_1,\dots,\beta_d \in \R$, and define
$\tdl_{\alpha_1,\dots,\alpha_d,\beta_1,\dots,\beta_d }(t)=\td(\alpha_1 t+\beta_1,\dots,\alpha_d t+\beta_d )$.
The function
$\tdl_{\alpha_1,\dots,\alpha_d,\beta_1,\dots,\beta_d }$ is  quasi-concave.
\end{claim}
\begin{proof}
Let $t_0 < t_1 < t_2$ and let $(a_1,\dots,a_d)$ be a direction that minimizes $\td(\alpha_1 t_1+\beta_1,\dots,\alpha_d t_1+\beta_d )$ in~(\ref{eq:td}),
i.e.,
$\tdl_{\alpha_1,\dots,\alpha_d,\beta_1,\dots,\beta_d }(t_1)=\td(\alpha_1 t_1+\beta_1,\dots,\alpha_d t_1+\beta_d )=t_{a_1,\dots,a_d}\left(\sum_{i=1}^d a_i (\alpha_i t_1+\beta_i) \right).$
Consider the function which maps $t$ to $\sum_{i=1}^d a_i (\alpha_i t +\beta_i)= (\sum_{i=1}^d a_i \alpha_i)t+\sum_{i=1}^d a_i \beta_i$. This function is either increasing or decreasing, thus, by Claim~\ref{cl:t-concave},
\begin{eqnarray*}
\tdl_{\alpha_1,\dots,\alpha_d,\beta_1,\dots,\beta_d }(t_1) & = &t_{a_1,\dots,a_d}\left(\sum_{i=1}^d a_i (\alpha_i t_1+\beta_i) \right)\\
& \geq &
\min\set{t_{a_1,\dots,a_d}\left(\sum_{i=1}^d a_i (\alpha_i t_0 +\beta_i) \right),
         t_{a_1,\dots,a_d}\left(\sum_{i=1}^d a_i (\alpha_i t_2+\beta_i) \right)} \\
& \geq &
\min\set{\td( \alpha_1 t_0 +\beta_1,\dots, \alpha_d t_0 +\beta_d),
         \td( \alpha_1 t_2 +\beta_1,\dots, \alpha_d t_2 +\beta_d)} \\
& = &
\min\set{\tdl_{\alpha_1,\dots,\alpha_d,\beta_1,\dots,\beta_d }(t_0),\tdl_{\alpha_1,\dots,\alpha_d,\beta_1,\dots,\beta_d }(t_1)}.
\end{eqnarray*}
\end{proof}
\begin{lemma}
For every $1 \leq i \leq d$ and every $x^*_1,\dots,x^*_{i-1} \in \R$, the function
$Q_{x^*_1,\dots,x^*_{i-1}}(x_i)$ is a quasi-concave function. Furthermore, $Q_{x^*_1,\dots,x^*_{i-1}}(x_i)$ is a sensitivity $1$ function of the multiset $S$.
\end{lemma}
\begin{proof}
Let $x_i^{0},x_i^1,x_i^2$ such that $x_i^0 < x_i^1 < x_i^2$.
Furthermore, let $x_{i+1}^{0},\dots,x_{d}^{0}$ and $x_{i+1}^{2},\dots,x_{d}^{2}$ be points
maximizing the functions $\td (x^*_1,\dots,x^*_{i-1},x_i^0,\cdot,\dots,\cdot)$ and $\td (x^*_1,\dots,x^*_{i-1},x_i^2,\cdot,\dots,\cdot)$ respectively, that is,
$Q_{x^*_1,\dots,x^*_{i-1}}(x_i^b)=\td(x^*_1,\dots,x^*_{i-1},x_i^b,x_{i+1}^{b},\dots,x_{d}^{b})$ for $b\in\{0,2\}$.

Consider the line $L:\R\rightarrow \R^d$ passing through the points $(x^*_1,\dots,x^*_{i-1},x_i^b,x_{i+1}^{b},\dots,x_{d}^{b})$ 
for $b\in\{0,2\},$
and scale its parameter such that $L(x_i^b)=(x^*_1,\dots,x^*_{i-1},x_i^b,x_{i+1}^{b},\dots,x_{d}^{b})$.
\remove{
That is, let 
$$\alpha_j=\left\{
\begin{array}{ll}
0 & \text{\rm if } 1\leq j \leq i-1 \\
1 & \text{\rm if } j = i \\
(x_j^0-x_j^2)/(x_i^0-x_i^2) & \text{\rm if } i+1 \leq j \leq d 
\end{array}
\right. $$ 
and 
$$\beta_j=\left\{
\begin{array}{ll}
x^*_j & \text{\rm if } 1\leq j \leq i-1 \\
0 & \text{\rm if } j = i \\
(x_i^0 x_j^2-x_i^2x_j^0)/(x_i^0-x_i^2) & \text{\rm if } i+1 \leq j \leq d 
\end{array}
\right. $$ 
and define $L(t)=(\alpha_1 t+\beta_1,\dots,\alpha_dt+\beta_d)$. 
}
%
{In particular, for every $x\in \R$ the $i$'th coordinate in $L(x)$ is $x$.}
By Claim~\ref{cl:tdl-concave} and the definition of $Q_{x^*_1,\dots,x^*_{i-1}}$,
$$Q_{x^*_1,\dots,x^*_{i-1}}(x_i^1) \geq \td(L(x_i^1)) 
\geq \min\set{ \td(L(x_i^0)), \td(L(x_i^2))}
=\set{ Q_{x^*_1,\dots,x^*_{i-1}}(x_i^0), Q_{x^*_1,\dots,x^*_{i-1}}(x_i^2)}.$$

The fact that $Q_{x^*_1,\dots,x^*_{i-1}}$ has sensitivity $1$ is implied by Observation~\ref{obs:Tukey} and the fact that maximum of sensitivity 1 functions is a sensitivity 1 function.
\end{proof}

\subsection{Extending the Domain}

The input to the private algorithm for finding a point in the convex hull is a dataset of points 
$S \subseteq X$, where  $X$ is a finite set whose size is at most $T$.
We note that the dataset $S$ may contain several copies of the same point (i.e.\ it is a multiset).
By the results of \cite{BNSV15}, the restriction to subsets of a finite set $X$ is essential (even when $d=1$).  

Notice that a Tukey median of $S$ might not be a point in $X$.
Furthermore, the proof that the functions $Q_{x^*_1,\dots,x^*_{i-1}}$ are  quasi-concave is over the reals.
Therefore, we extend the domain to $\tilde{X}=\prod_{i=1}^d\tilde{X_i}$ 
such that for every dataset $S$ the functions $Q_{x^*_1,\dots,x^*_{i-1}}$ attain their maximum over the extended domain. 
We will not try to optimize the size of $\tilde{X_1}, \ldots, \tilde{X_d}$ 
as the dependency of the sample complexity of~$\AlgRecConcave$ on $|\tilde{X_i}|$ is~$2^{O(\log^* |\tilde{X_i}|)}$.  
\begin{claim}
\label{cl:sets}
There exists sets $\tilde{X_1}, \ldots, \tilde{X_d}$ such that
$|\tilde{X_i}| \leq (dT^{d^2(d+1)})^{2^d}$ for $1 \leq i \leq d$ and
for every dataset $S$, 
for every $1 \leq i \leq d$, and for every $x^*_1,\dots,x^*_{i-1}\in \tilde{X_1}\times \dots \times \tilde{X_{i-1}}$, there exist  
$x^{m}_i,\dots,x^{m}_{d}\in \tilde{X_i} \times \ldots \times \tilde{X_{d}}$ such that
\begin{equation}
\label{eq:ExistsMax}
\max_{x_{i+1},\dots,x_d \in \R} \td (x^*_1,\dots,x^*_{i-1}, x_i,\dots,x_d)
=\td (x^*_1,\dots,x^*_{i-1}, x^m_i,\dots,x^m_d).
\end{equation}
\end{claim}
\begin{proof}
For $1\leq i \leq d$, let $X_i$ be the projection of $X$ to the $i$th coordinate, that is, $$X_i=\set{x: \exists_{x_1,\dots,x_{i-1},x_{i+1},\dots,x_d} (x_1,\dots,x_{i-1},x,x_{i+1},\dots,x_d) \in X }.$$

The construction heavily exploits \cref{fact:counting}.
Let $L$ denote the set of all affine subspaces that are spanned by points in $X_1\times\cdots\times X_d$.
Since each such subspace is spanned by at most $d+1$ points, it follows that~$\lvert L \rvert \leq {T^d \choose d+1}\leq T^{d(d+1)}$.
By \cref{fact:counting}, for every dataset $S$ and every $r>0$, 
every vertex of $\T(r)$ can be written as the intersection of at most~$d$ subspaces in~$L$. 
In particular, there exists a Tukey median that is the intersection of at most~$d$ subspaces in~$L$.

We construct the sets in iterations where 
we start with $\tilde{X_{j}}=\emptyset$ for every $1 \leq j \leq d$.
In iteration $i$ we do the following: for every $x^*_1,\dots,x^*_{i-1}\in \tilde{X}_1\times \dots \times \tilde{X}_{i-1}$
and for every $d-i$ subspaces in $L$ such that there exists a unique point $x = (x^*_1\ldots x^*_{i-1},x_i\ldots,x_d)$
in the intersection of these $d-i$ subspaces, we add $x_j$ to $\tilde X_j$ for all $j\geq i$.

We next argue that item (ii) in the conclusion of the claim is satisfied:
indeed, by \cref{fact:counting}, this construction contains a vertex of every set of the form 
$\T(r) \cap \{x\in\R^d: x_1 = x^*_1,\ldots x_{i-1}=x^*_{i-1}\}$,
for every $r>0$, $i\leq d$, and every $(x^*_1,\ldots,x^*_{i-1})\in \tilde X_1\times\ldots\times \tilde X_{i-1}$.
In particular, by plugging 
\[r= \max_{x_{i+1},\dots,x_d \in \R} \td (x^*_1,\dots,x^*_{i-1}, x_i,\dots,x_d),\]
it contains a point which satisfies \cref{eq:ExistsMax}.
This implies item (ii).

As for item (i), note that the size of $\tilde X_1$ is at most ${\lvert L\rvert \choose d}\leq T^{d^2(d+1)}$.
Similarly, for $i > 1$:
\[\lvert \tilde X_i\rvert\leq {\lvert L\rvert \choose d} + {\lvert L\rvert \choose d-1}\lvert \tilde X_1\rvert + \ldots 
+ {\lvert L\rvert \choose d-i}\prod_{j=1}^{i-1}\lvert \tilde X_j\rvert \leq \Bigl( d\cdot {\lvert L\rvert \choose d}\Bigr)^{2^d}\leq (dT^{d^2(d+1)})^{2^d}.\]
\end{proof}

\subsection{The Algorithm}
In \cref{fig:Tukey}, we present an $(\eps,\delta)$-differentially private algorithm $\AlgFindTukey$ that with probability at least $1-\beta$ finds a point whose Tukey depth is at least $n/2(d+1)$. The informal description of the algorithm appears in the beginning of \cref{sec:FindingPoint}.
\begin{figure}[thb!]
\begin{center}
\noindent\fbox{
\parbox{.95\columnwidth}{
\begin{center}{ \bf Algorithm $\AlgFindTukey$}\end{center}
{\bf Preprocessing:}
\begin{itemize}
\item Construct the sets $\tilde{X}_1,\ldots,\tilde{X}_d$ as  in Claim~\ref{cl:sets}. Let $\tilde{T}=\max_{1\leq i \leq d} |\tilde{X_i}|$. 
\\$(*$ By Claim~\ref{cl:sets}, $\log^* \tilde{T}=\log^* d +\log^*T +O(1)$. $*)$
\end{itemize}

{\bf Algorithm:} 
\begin{itemize}
\item [(i)] Let $\beta,\eps,\delta$ be the utility/privacy parameters, and $S$ be an input database from~$X$.
\item [(ii)]
For $i=1$ to $d$ do:
\begin{itemize}
\item [(a)]
For every $x_i \in \tilde{X}_i$ define $$Q_{x^*_1,\dots,x^*_{i-1}}(x_i)\triangleq \max_{x_{i+1}\in \tilde{X}_{i+1},\dots,x_d \in \tilde{X}_d} \td (x^*_1,\dots,x^*_{i-1}, x_i,\dots,x_d).$$
\item [(b)]
Execute $\AlgRecConcave$ on $S$ with the function $Q_{x^*_1,\dots,x^*_{i-1}}$ and parameters $r=\frac{n}{d+1}-\frac{(i-1)n}{d(d+1)}$, $\alpha_0=\frac{1}{2d},\beta_0=\frac{\beta}{d},\eps_0=\frac{\eps}{2\sqrt{2d\ln(2/\delta)}},\delta_0=\frac{\delta}{2d}$. Let $x^*_i$ be its output.
\end{itemize}
\item [(iii)]
Return $x^*_1,\dots,x^*_d$.\\
\end{itemize}
}}
\end{center}
\caption{Algorithm $\AlgFindTukey$ for finding a point whose Tukey depth is at least $n/2(d+1)$.\label{fig:Tukey}}
\end{figure}

\begin{theorem}\label{thm:privatecenterpoint}
Let $ \eps \leq 1$ and $\delta < 1/2$ and $X \subset\R^d$ be a set of size at most $T$.
Assume that the input dataset $S\subseteq X$ satisfies 
\[|S| =O\Biggl(d^{2.5}\cdot 2^{O(\log^*T +\log^*d)}\frac{\log^{0.5}\bigl(\frac{1}{\delta}\bigr) \log \bigl(\frac{d^2}{\beta\delta}\bigr)}{\eps}\Biggr).\]
Then, $\AlgFindTukey$ is an $(\eps,\delta)$-differentially private algorithm
that with probability at least $1-\beta$ returns a point $x^*_1,\dots,x^*_d$ such that 
$\td(x^*_1,\dots,x^*_d)\geq\frac{\lvert S\rvert}{2(d+1)}$.  
\end{theorem}

\begin{proof}
The proof of the correctness (utility) of $\AlgFindTukey$ is proved by induction,
using the correctness of $\AlgRecConcave$. The privacy proof follows
from the privacy of $\AlgRecConcave$ and using the advanced composition theorem~(\cref{thm:composition2}).

\paragraph{Utility.}
We prove by induction that after step $i$ of the algorithm, with probability at least $1-i\beta/d$, 
the returned values $x^*_1,\ldots,x^*_i$ satisfy
$Q_{x^*_1,\dots,x^*_{i-1}}(x^*_i)\geq \frac{\lvert S\rvert}{d+1}(1-\frac{i}{2d})$,  i.e., there are $(x_{i+1},\ldots,x_d)\in \tilde{X}_{i+1}\times\dots\times\tilde{X}_{d}$ 
such that $\td(x^*_1,\dots,x^*_{i},x_{i+1},\dots,x_d)\geq \frac{\lvert S\rvert}{d+1}(1-\frac{i}{2d})$.

The basis is the induction is $i=0$: by Claim~\ref{c:centerpoint} the Tukey median has depth at least~$\lvert S\rvert/(d+1)$
and by  Claim~\ref{cl:sets}, the median is in $\tilde{X}_{1}\times \cdots \times\tilde{X}_{d}$.
Thus, with probability $1$ there are $(x_{1},\ldots,x_d)\in \tilde{X}_{1}\times\dots\times\tilde{X}_{d}$ 
such that $\td(x_{1},\dots,x_d)\geq \frac{\lvert S\rvert}{d+1}$.


Next, by the induction hypothesis for $i-1$,
with probability at least $1-(i-1)\beta/d$ it holds that
\[\max_{x\in\tilde{X_i}}\{Q_{x^*_1,\dots,x^*_{i-1}}(x)\}\geq 
\frac{\lvert S\rvert}{d+1}-\frac{(i-1)\lvert S\rvert}{2d(d+1)}=r > \frac{\lvert S\rvert}{2(d+1)}
 \geq
8^{\log^* \tilde{T}} \cdot \frac{12 \log^* \tilde{T}}{\alpha_0\eps_0}\log\Big(\frac{192(\log^* T)^2}{\beta_0\delta_0}\Big).
\]
Therefore, by \cref{prop:aRecConcave}, with probability at least 
$(1-\beta/d)\bigl(1-(i-1)\beta/d\bigr)\geq 1 - i\beta/d$ Algorithm
$\AlgRecConcave$ returns $x^*_i\in \tilde{X}_i$ such that 
$$Q_{x^*_1,\dots,x^*_{i-1}}(x^*_i)\geq (1-\alpha)r=\left(1-\frac{1}{2d}\right)\frac{\lvert S\rvert}{d+1}\left(1-\frac{i-1}{2d}\right)>\frac{\lvert S\rvert}{d+1}\left(1-\frac{i}{2d}\right).$$

To conclude, after $d$ steps of the algorithm, $\td(x^*_1,\dots,x^*_d)\geq \frac{\lvert S\rvert}{2(d+1)}$ with probability at least $1-\beta$. 

\paragraph{Privacy.}
By \cref{prop:aRecConcave}, each invocation of $\AlgRecConcave$ is
$(\eps_0,\delta_0)$-differentially private. $\AlgFindTukey$ 
invokes $\AlgRecConcave$ $d$ times. 
Thus, by \cref{thm:composition2} (the advanced composition) with $\delta'=\delta/2$,
it follows that $\AlgFindTukey$ is $(\frac{\eps}{2}+\frac{\eps^2}{4\ln (2/\delta)},\delta)$ differentially-private,
which implies $(\eps,\delta)$-privacy whenever $\eps \leq 1$ and~$\delta \leq 1/2$.
%
%
\end{proof}

\section{Learning Halfspaces Using Convex Hull} 

We describe in \cref{fig:reduction} a reduction from learning halfspaces to finding a point in a convex-hull of a {multiset} of points. Furthermore, we show that if  
the algorithm we use in the reduction finds a point whose Tukey depth is high (as our algorithm from \cref{sec:FindingPoint} does), then the required sample complexity of the learning algorithm is reduced. As a result, we get  
an upper bound of ${\tilde{O}(d^{4.5}2^{\log^*|X| })}$ on the sample complexity of private learning halfspaces (ignoring the privacy and learning parameters). 
In comparison, using the exponential mechanism of~\cite{MT07} results in an $(\eps,\delta)$-deferentially private algorithm whose sample complexity is $O(d \log |X|)$, e.g., for the interesting case where $X=[T]^d$ for some $T$, the complexity is $O(d^2 \log T)$. 
Our upper bound is better than the sample complexity of the exponential mechanism when $d$ is small compared to $\log |T|$, in particular when $d$ is constant.


\begin{figure}[htb!]
\begin{center}
\noindent\fbox{
\parbox{.95\columnwidth}{
\begin{center}{ \bf Algorithm $\AlgHalfSpace$}\end{center}
{\bf Preprocessing:}
\begin{itemize}
\item Fix a set~$H\subseteq\R^{d+1}$ that contains representations of all halfspaces in $\halfspace(X)$, as in Claim~\ref{c:grid}.
\end{itemize}
{\bf Algorithm:}
\begin{enumerate}
\item
Let $\eps,\delta,\alpha,\beta$ be the privacy and utility parameters and let $S$ be a realizable input sample of size~$s$, 
where $s$ is as in \Cref{thm:convexhullreduction}. 
\item  \label{step:partition} Partition $S$ into $m$
equisized subsamples $S_1,\ldots,S_m$, where $m=m(d+1,2|X|^{d+1},\eps,\delta,\beta/2)$ as in \cref{thm:convexhullreduction}.\\
$(*$ Note that each $S_i$ has size $\Theta\bigl(\frac{d\log(\frac{m}{r\alpha}) + \log(2m/\beta)}{r\alpha/m}\bigr)$. $*)$
\item For each $S_i$ pick a consistent halfpace $h_i\in H$ uniformly at random.
\item \label{step:Alg} Apply an $(\eps,\delta)$-differentially private algorithm $\Alg$ for finding a point in a convex hull with parameters $\eps,\delta,\frac{\beta}{2}$ on $H_0=(h_1\ldots h_m)$.
\item Output the halfspace $h$ found by $\Alg$.
\end{enumerate}
}}
\end{center}
\caption{\label{fig:reduction}A reduction from learning halfspaces to finding a point in a convex hull.}
\end{figure}

We start by showing the existence of a set $H$ that is used by the algorithm. We say that a vector $(a_1,\ldots ,a_d,w)\in\R^{d+1}$ {\it represents} a halfspace $\hs \in \halfspace(X)$
if $\hs(\pt{x})=\hs_{a_1,\ldots, a_d,w}(\pt{x})$ for every $\pt{x} \in X$.
Note that every $\hs \in \halfspace(X)$ has many representations.

\begin{claim}
\label{c:grid}
There exists a set $H \subseteq \R^{d+1}$, where $\lvert H\rvert \leq 2\lvert X \rvert^{d+1}$
which contains  one representation of each halfspace $\hs\in\halfspace(X)$.
\end{claim}
\begin{proof}
By standard 
bounds from discrete geometry, $\lvert \halfspace(X) \rvert \leq 2\lvert X \rvert^{d+1}$ (see, e.g.~\cite{Gartner94vapnik}).
For each $\hs\in \halfspace(X)$ pick a representation $(a_1\ldots a_d,w)\in \R^{d+1}$.
\end{proof}

\begin{theorem}\label{thm:convexhullreduction}
Assume that Algorithm $\Alg$ used in step \ref{step:Alg} of Algorithm $\AlgHalfSpace$ is an $(\eps,\delta)$-differentialy private algorithm  that finds with probability at least $1-\beta$ a point in a convex hull for a multisets $S \subseteq X \subset \R^d$ whose Tukey depth is at least $r$ provided that 
$|S| \geq m(d,|X|,\eps,\delta,\beta)$ for some function $m(\cdot,\cdot,\cdot,\cdot,\cdot)$.

Let $\eps\leq 1, \delta\leq \frac{1}{2}$ and $\alpha,\beta\leq 1$ be the privacy and utility parameters.
Then, $\AlgHalfSpace$ is an $(\eps,\delta)$-differentially private $(\alpha,\beta)$-PAC learner with sample complexity $s$ for the class $\halfspace(X)$ for
\[
 s=O\Bigl( \frac{m^2\cdot d\log(\frac{m}{r\alpha}) + \log(m/\beta)}{r\alpha} \Bigr)\]
where $m=m(d+1,2|X|^{d+1},\eps,\delta,\beta/2)$.
\end{theorem}

\begin{proof}
We first establish the privacy guarantee of the algorithm 
and later argue that it PAC-learns $\halfspace(X)$.

\paragraph{Privacy.}
Let $S^1,S^2$ be two neighboring input samples of size at least $s$, where $s$ is as in the theorem statement.
Let $H_0(S^1),H_0(S^2)$ denote the list of halfspaces that are derived in step~\ref{step:partition} of the algorithm
when it is applied on $S^1,S^2$.
Since the $h_i$'s are constructed from mutually disjoint subsamples
it follows that the datasets $H_0(S_1)$ and $H_0(S_2)$ are neighbors.
The $(\eps,\delta)$-privacy guarantee now follows from the privacy gurantee of $\Alg$
since the size of the $H_0(S^i)$'s is $m=m(d+1,2|X|^{d+1},\eps,\delta,\beta/2)$.

\paragraph{Utility.}
We next establish that the algorithm learns $\halfspace(X)$ with confidence $1-\beta$ and error $\alpha$.
Let~$\DDD$ denote the target distribution and $c\in\halfspace(X)$ denote the target concept.
Let $S$ denote the input sample of size at least $s$ that is sampled independently from $\DDD$ and labeled by $c$.

We first claim that every halfspace $h_i\in H_0$ the probability  $h_i$ has error greater than $\frac{r\alpha}{m}$ with respect to the distribution $\DDD$ is at most $\beta/2m$.
This follows directly from standard bounds on the (non-private) sample complexity of PAC learning of VC classes,
since the VC dimension of $\halfspace(X)$ is at most $d+1$
and since each $S_i$ has size $\Omega(\frac{d\log(1/\alpha') + \log(1/\beta')}{\alpha'})$
where $\alpha'=\frac{r\alpha}{m}$ and $\beta'=\frac{\beta}{2m}$ (see e.g.\ Theorem 6.8 in~\cite{Shalev14book}).

We next claim that if $h \in \R^{d+1}$ errs on a point $\pt{x}=(x_1,\dots,x_d)$ then at least $r$ halfspaces in $H_0$ err on $\pt{x}$.
By the assumption in the theorem, $\Alg$ outputs with probability at least $1-\beta/2$ a halfspace $h$ whose Tukey depth is at least $r$ with respect to $H_0$.
By duality, the set of all halfspaces that err on $\pt{x}$ is
itself a halfspace in $\R^{d+1}$ that contains all points
$(a_1,\dots,a_d,w)\in \R^d$ such that $\mathsf{sign}(\sum_{i=1}^d x_i a_i - w) \neq c(x)$. Denote this halfspace by~$h_{\text{err}}$.
By assumption, $h\in h_{\text{err}}$. 
Thus, since the Tukey Depth of $h$ with respect to~$H_0$ is at least $r$,
at least $r$ of the halfspaces in $H_0$ are in $h_{\text{err}}$,
as required.

We are ready to establish the PAC-learning guarantee.
Assume that $\error_{\DDD}(c,h_i) \leq r\alpha/m$ for every  $1\leq i \leq m$ and that $\Alg$ returns a point whose Tukey rank is at most $r$.
By the argument above and the union bound, this happens with probability at least $1-\beta$.
Let $E_h:X\to\{0,1\}$ denote the indicator the $h$ errs (i.e.\ $E_h(\pt{x})=1$ if and only if $h(\pt{x})\neq c(\pt{x})$).
Similarly, let $E_{h_i}$ denote the indicator that $h_i\in H_0$ errs.
For every $\pt{x}\in X$:
$
 E_h(\pt{x}) \leq \frac{1}{r}\sum_{i=1}^m E_{h_i}(\pt{x}) 
$
(either $E_h(\pt{x})=0$ or $E_h(\pt{x})=1$ and $\sum_{i=1}^m E_{h_i}(\pt{x})\geq r$). 
Therefore, by taking expectation over both sides it follows that with probability at least $1-\beta$:\;
$\error_{\DDD}(c,h) = \E_{\pt{x}\sim\DDD}[E_h(\pt{x})] \leq \E_{\pt{x}\sim\DDD}\Bigl[\frac{1}{r}\sum_{i=1}^m E_{h_i}(\pt{x})\Bigr]=
\frac{1}{r} \sum_{i=1}^m\error_{\DDD}(c,h_i) \leq \frac{m}{r}\cdot(r\alpha/m)= \alpha,$
as required.
\end{proof}

Using $\AlgFindTukey$ in step~\ref{step:Alg} of Algorithm $\AlgHalfSpace$,
we get the following corollary (which follows from \cref{thm:privatecenterpoint} and \cref{thm:convexhullreduction}).
\begin{corollary}
\label{cor:halfspacelearner}
Let $\eps \leq 1$, $\delta <1/2$, and $X\subseteq \R^d$ be a set.
There exists an $(\eps,\delta)$-differentially private $(\alpha,\beta)$-PAC learner with sample complexity $s$ for $\halfspace(X)$ 
with
\[
s =  \tilde{O}\Bigl(\frac{d^{4.5}2^{O(\log^*|X| +\log^*d)}\log^{1.5}\frac{1}{\delta} \log^2 \frac{1}{\beta}}{\eps\alpha} \Bigr).
\]
\end{corollary}

Corollary~\ref{cor:halfspacelearner} establishes an upper bound on the sample complexity of privately learning halfspaces
whose dependency on the domain size $|X|$ is $2^{O\log^*(|X|)}$.
The crux of the algorithm is a reduction to privately publishing a point with a large Tukey depth with respect to a given input dataset.
A drawback of this approach is that the latter task is likely to be computationally difficult (even without privacy constraints), unless the dimension $d$ is constant
 (see \cite{Miller10approx} and references within).

In $\AlgHalfSpace$ we can use an algorithm $\Alg$  that finds a point in the convex hull (i.e., a point whose Tukey depth is at least $1$).
The resulting learning algorithm require sample complexity of $O(\frac{m^2\cdot d\log(\frac{m}{\alpha}) + \log(m/\beta)}{\alpha})$,
where $m$ is the sample complexity of $\Alg$. This may result in a more efficient private learning algorithm for halfspaces as the task of privately finding a point in the convex hull might be easier than the task of privately finding a point with high Tukey degree. Furthermore, in this case, we can use an algorithm that privately finds a hypothesis that is a linear combination with positive coefficients of the hypotheses in $H_0$. This follows from the observation that if all hypotheses in $H_0$ are correct on a point $\pt{x}$, then any  linear combination with positive coefficients of the hypotheses in $H_0$ is correct of $\pt{x}$.

\section{A Lower Bound on the Sample Complexity of Privately Finding a Point in the Convex Hull}

In this section we show a lower bound on the sample complexity of privately finding a point in the convex hull of a database $S\subseteq X = [T]^d$.
We show that any $(\eps,\delta)$-differentially private algorithm for this task must have sample complexity $\Omega(\frac{d}{\eps}\log \frac{1}{\delta})$.
Our lower bound actually applies to a possibly simpler task of finding a
non-trivial linear combination of the points in the database.

By \cite{BNSV15}, finding a point in the convex hull (even for $d=1$) requires sample complexity $\Omega(\log^* T)$. Thus, together we get a lower bound on the sample complexity of
$\Omega(\frac{d}{\eps}\log\frac{1}{\delta}+\log^* T)$.

It may be tempting to guess that, even with pure $(\eps,0)$-differential privacy, a sample complexity of $O(\log |X|) = O(d\log T)$ should suffice for solving this task, as the size of the output space is $T^d$, because $S\subseteq [T]^d$, and hence (it seems) that one could privately solve this problem using the exponential mechanism of {\cite{MT07}} with sample complexity that depends logarithmically on the size of the output space. We show that this is not the case, and that any $(\eps,0)$-differentially private algorithm for this task must have sample complexity $\Omega(\frac{d^2}{\eps}\log T)$. 


\begin{theorem}\label{thm:lowerBound}
Let $T\geq2$, and $d\geq10$. Let $\AAA$ be an $(\eps,\delta)$-differentially private algorithm that takes a database $S\subseteq [T]^d$ of size $m$ and returns, with probability at least $1/2$, 
a non-trivial linear combination of the points in $S$. Then,
$$m=\Omega\left(
\min\left\{\frac{d^2}{\eps}\log T,\; 
\frac{d}{\eps}\log\frac{1}{\delta}
\right\}\right).$$
\end{theorem}

The proof of Theorem~\ref{thm:lowerBound} builds on the analysis of~\cite{BLR08} for lower bounding the sample complexity of releasing approximated answers for counting queries. 

\begin{proof}
Throughout this proof, we use $\vspan(S)$ to denote the set of all non-trivial linear combinations of the points in $S$. 
Let $I=\{\pt{x_1},\pt{x_2},\dots,\pt{x_{d/2}},\pt{x'}\}$ be a multiset of random points, where each point is chosen independently and uniformly from $[T]^d$. Also let $i$ be chosen uniformly from $\{1,2,\dots,d/2\}$. Now define the database $S$ containing $\frac{2m}{d}$ copies of each of $\pt{x_1},\dots,\pt{x_{d/2}}$ and define the database $S'$ containing $\frac{2m}{d}$ copies of each of $\pt{x_1},\dots, \pt{x_{i-1}},\pt{x_{i+1}},\dots \pt{x_{d/2}},\pt{x'}$. Note that $S,S'$ differ in exactly $\frac{2m}{d}$ points.


Observe that the points in $I$ are linearly independent with high probability. To see this note that any set $V$
of size at most $\frac{d}{2}$ spans at most $T^{d/2}$ vectors in $[T]^d$: indeed, without loss of generality
we may assume that $V$ is independent and therefore
can be completed to a {\em basis} $\R^d$ by adding $d-\lvert V\rvert \geq d/2$ unit vectors (since the dimension of $V$ is at most $d/2$ there is at least one unit vector that it does not span, add this vector to $V$ and continue). Without loss of generality, these unit vectors are $\pt{e_1},\dots,\pt{e_{d/2}}$. Thus, every choice from $[T]^{d/2}$ for the last $d/2$ coordinates can be completed in a most one way to a vector spanned by $V$ (because every element of $\R^d$ may be written in a {\em unique} way as a linear combination of elements of the resulting basis). This means that a set of (at most) $d/2$ points spans at most $T^{d/2}$ vectors in $[T]^d$. Hence, by a union bound, the probability that the points in $I$ are not independent is at most $\frac{d}{2}\cdot T^{-d/2}$.


Let $\BBB(\pt{b},I)$ be a procedure that operates on a point $\pt{b}$ and a set of points $I$, defined as follows. If $I$ is not linearly independent, or if $\pt{b}\notin\vspan(I)$, than the procedure outputs $\bot$. Otherwise the procedure returns the point $\pt{x}\in I$ with the largest coefficient when representing $\pt{b}$ as a linear combination of the points in $I$ (ties are broken arbitrarily).
Observe that if $\pt{b}\in\vspan(S)$ for a subset $S\subseteq I$ and if the points in $I$ are linearly independent then $\BBB(\pt{b},I)\in S$. Let $\beta$ denote the probability that $\AAA(S)$ fails to return a point in $\vspan(S)$.  As $i$ is uniform on $\{1,2,\dots,d/2\}$ we have
\begin{align*}
\Pr_{I,i,\AAA}[\BBB( \AAA(S), I )=\pt{x_i}]
&\geq \Pr_{I,\AAA}\left[
\begin{array}{l}
I \text{ is independent}, \\
\AAA(S)\in\vspan(S)
\end{array}\right]
\cdot \Pr_{I,i,\AAA}\left[\BBB( \AAA(S), I )=\pt{x_i} \left| 
\begin{array}{l}
I \text{ is independent}, \\
\AAA(S)\in\vspan(S)
\end{array}
\right.\right]\\
&\geq\left(1-\beta-\frac{d}{2}\cdot T^{-d/2}\right)\cdot \frac{2}{d}
\geq\frac{1}{2d}\;,
\end{align*}
where {the second inequality is implied by the fact that $i$ is chosen with uniform distribution from a set of size $d/2$ and} the last inequality is by asserting that $\beta\leq1/2$, $T\geq2$, and $d\geq10$.

On the other hand observe that if $\BBB( \AAA(S'), I )=\pt{x_i}$ then
%
	(1) $\AAA(S')\in\vspan(I)$, as otherwise $\BBB$ outputs $\bot$,
	(2) $I$ is linearly independent, as otherwise $\BBB$ outputs $\bot$, and
	(3) $\AAA(S')\notin\vspan(I\setminus\{\pt{x_i}\})$, as otherwise the coefficient of $\pt{x_i}$ in $\AAA(S')$ is 0, and $\BBB$ will not output $\pt{x_i}$.

Let $I_{-i}=I\setminus\{\pt{x_i}\}{=\set{\pt{x_1},\dots,\pt{x_{i-1}},\pt{x_{i+1}},\dots,\pt{x_{d/2},\pt{x'}}}}$ and $\pt{b}\leftarrow\AAA(S')$. 
We have that
\begin{align*}
\Pr_{I,i,\AAA}[\BBB( \pt{b}, I )=\pt{x_i}]&\leq\Pr_{I,i,\AAA}[\pt{b}\in\vspan(I) \text{ and } \pt{b}\notin\vspan(I_{-i}) \text{ and } I \text{ is independent} ]\\
&\leq\Pr_{I,i,\AAA}[\pt{x_i}\in\vspan(I_{-i}\cup\{\pt{b}\})] 
\, \leq \, T^{-d/2},
\end{align*}
where the 
last inequality is because $\pt{x_i}$ is independent of $I_{-i}$ and $\pt{b}$ (recall that we denoted $\pt{b}\leftarrow\AAA(S')$, and hence, $\pt{b}$ is a (random) function of $I_{-i}$, which is independent of $\pt{x_i}$). 
Therefore, by the privacy guarantees of $\AAA$ we get
\begin{align*}
\frac{1}{2d}&\leq\Pr_{I,i,\AAA}[\BBB( \AAA(S), I )=\pt{x_i}]=\sum_{I,i}\Pr[I,i]\cdot\Pr_{\AAA}[\BBB( \AAA(S), I )=\pt{x_i}]\\
&\leq\sum_{I,i}\Pr[I,i]\cdot\left(e^{2\eps m/d}\cdot\Pr_{\AAA}[\BBB( \AAA(S'), I )=\pt{x_i}]+e^{2\eps m/d}\cdot 2\delta m/d\right)\\
&=e^{2\eps m/d}\cdot\Pr_{I,i,\AAA}[\BBB( \AAA(S'), I )=\pt{x_i}]+e^{2\eps m/d}\cdot 2\delta m/d\\
&\leq e^{2\eps m/d}\cdot T^{-d/2}+e^{2\eps m/d}\cdot 2\delta m/d.
\end{align*}
Solving for $m$, this means that 
$m=\Omega(
\min\{\frac{d^2}{\eps}\log T,\; 
\frac{d}{\eps}\log\frac{1}{\delta}
\}.$
\end{proof}

\bibliographystyle{abbrvnat}
\bibliography{pacparity}

\begin{thebibliography}{34}
\providecommand{\natexlab}[1]{#1}
\providecommand{\url}[1]{\texttt{#1}}
\expandafter\ifx\csname urlstyle\endcsname\relax
  \providecommand{\doi}[1]{doi: #1}\else
  \providecommand{\doi}{doi: \begingroup \urlstyle{rm}\Url}\fi

\bibitem[Alon et~al.(2018)Alon, Livni, Malliaris, and Moran]{ALMM18}
N.~Alon, R.~Livni, M.~Malliaris, and S.~Moran.
\newblock Private {PAC} learning implies finite littlestone dimension.
\newblock \emph{CoRR}, abs/1806.00949, 2018.
\newblock URL \url{http://arxiv.org/abs/1806.00949}.

\bibitem[Bassily et~al.(2018)Bassily, Thakurta, and Thakkar]{BassilyTT18}
R.~Bassily, A.~G. Thakurta, and O.~D. Thakkar.
\newblock Model-agnostic private learning.
\newblock In S.~Bengio, H.~M. Wallach, H.~Larochelle, K.~Grauman,
  N.~Cesa{-}Bianchi, and R.~Garnett, editors, \emph{Advances in Neural
  Information Processing Systems 31: Annual Conference on Neural Information
  Processing Systems 2018, NeurIPS 2018, 3-8 December 2018, Montr{\'{e}}al,
  Canada.}, pages 7102--7112, 2018.
\newblock URL
  \url{http://papers.nips.cc/paper/7941-model-agnostic-private-learning}.

\bibitem[Beimel et~al.(2013{\natexlab{a}})Beimel, Nissim, and Stemmer]{BNS13}
A.~Beimel, K.~Nissim, and U.~Stemmer.
\newblock Characterizing the sample complexity of private learners.
\newblock In \emph{ITCS}, pages 97--110. ACM, 2013{\natexlab{a}}.

\bibitem[Beimel et~al.(2013{\natexlab{b}})Beimel, Nissim, and Stemmer]{BNS13b}
A.~Beimel, K.~Nissim, and U.~Stemmer.
\newblock Private learning and sanitization: Pure vs. approximate differential
  privacy.
\newblock In \emph{APPROX-RANDOM}, volume 8096 of \emph{Lecture Notes in
  Computer Science}, pages 363--378. Springer, 2013{\natexlab{b}}.

\bibitem[Beimel et~al.(2014)Beimel, Brenner, Kasiviswanathan, and
  Nissim]{BBKN12}
A.~Beimel, H.~Brenner, S.~P. Kasiviswanathan, and K.~Nissim.
\newblock Bounds on the sample complexity for private learning and private data
  release.
\newblock \emph{Machine Learning}, 94\penalty0 (3):\penalty0 401--437, 2014.

\bibitem[Beimel et~al.(2015)Beimel, Nissim, and Stemmer]{BNS15}
A.~Beimel, K.~Nissim, and U.~Stemmer.
\newblock Learning privately with labeled and unlabeled examples.
\newblock In \emph{SODA}, pages 461--477. {SIAM}, 2015.

\bibitem[Beimel et~al.(2016)Beimel, Nissim, and Stemmer]{BNS16a}
A.~Beimel, K.~Nissim, and U.~Stemmer.
\newblock Private learning and sanitization: Pure vs. approximate differential
  privacy.
\newblock \emph{Theory of Computing}, 12\penalty0 (1):\penalty0 1--61, 2016.
\newblock URL \url{https://doi.org/10.4086/toc.2016.v012a001}.

\bibitem[Ben{-}David and Litman(1998)]{BL98}
S.~Ben{-}David and A.~Litman.
\newblock Combinatorial variability of vapnik-chervonenkis classes with
  applications to sample compression schemes.
\newblock \emph{Discrete Applied Mathematics}, 86\penalty0 (1):\penalty0 3--25,
  1998.
\newblock \doi{10.1016/S0166-218X(98)00000-6}.
\newblock URL \url{https://doi.org/10.1016/S0166-218X(98)00000-6}.

\bibitem[Blum et~al.(2005)Blum, Dwork, McSherry, and Nissim]{BDMN05}
A.~Blum, C.~Dwork, F.~McSherry, and K.~Nissim.
\newblock Practical privacy: The {SuLQ} framework.
\newblock In C.~Li, editor, \emph{PODS}, pages 128--138. ACM, 2005.

\bibitem[Blum et~al.(2008)Blum, Ligett, and Roth]{BLR08}
A.~Blum, K.~Ligett, and A.~Roth.
\newblock A learning theory approach to non-interactive database privacy.
\newblock In \emph{STOC}, pages 609--618. ACM, 2008.

\bibitem[Bun(2016)]{Bun16}
M.~Bun.
\newblock \emph{New Separations in the Complexity of Differential Privacy}.
\newblock PhD thesis, Harvard University, 2016.
\newblock Supervisor-Salil Vadhan.

\bibitem[Bun and Zhandry(2016)]{BZ16}
M.~Bun and M.~Zhandry.
\newblock Order-revealing encryption and the hardness of private learning.
\newblock In E.~Kushilevitz and T.~Malkin, editors, \emph{Theory of
  Cryptography - 13th International Conference, {TCC} 2016-A, Tel Aviv, Israel,
  January 10-13, 2016, Proceedings, Part {I}}, volume 9562 of \emph{Lecture
  Notes in Computer Science}, pages 176--206. Springer, 2016.
\newblock ISBN 978-3-662-49095-2.
\newblock \doi{10.1007/978-3-662-49096-9\_8}.
\newblock URL \url{https://doi.org/10.1007/978-3-662-49096-9\_8}.

\bibitem[Bun et~al.(2015)Bun, Nissim, Stemmer, and Vadhan]{BNSV15}
M.~Bun, K.~Nissim, U.~Stemmer, and S.~P. Vadhan.
\newblock Differentially private release and learning of threshold functions.
\newblock In \emph{{FOCS}}, pages 634--649, 2015.

\bibitem[Bun et~al.(2016)Bun, Nissim, and Stemmer]{BNS16}
M.~Bun, K.~Nissim, and U.~Stemmer.
\newblock Simultaneous private learning of multiple concepts.
\newblock In \emph{ITCS}, pages 369--380. ACM, 2016.

\bibitem[Chaudhuri and Hsu(2011)]{CH11}
K.~Chaudhuri and D.~Hsu.
\newblock Sample complexity bounds for differentially private learning.
\newblock In S.~M. Kakade and U.~von Luxburg, editors, \emph{COLT}, volume~19
  of \emph{JMLR Proceedings}, pages 155--186. JMLR.org, 2011.

\bibitem[Dunagan and Vempala(2008)]{DV08}
J.~Dunagan and S.~Vempala.
\newblock A simple polynomial-time rescaling algorithm for solving linear
  programs.
\newblock \emph{Mathematical Programming}, 114\penalty0 (1):\penalty0 101--114,
  Jul 2008.
\newblock ISSN 1436-4646.

\bibitem[Dwork and Feldman(2018)]{DworkF18}
C.~Dwork and V.~Feldman.
\newblock Privacy-preserving prediction.
\newblock In S.~Bubeck, V.~Perchet, and P.~Rigollet, editors, \emph{Conference
  On Learning Theory, {COLT} 2018, Stockholm, Sweden, 6-9 July 2018.},
  volume~75 of \emph{Proceedings of Machine Learning Research}, pages
  1693--1702. {PMLR}, 2018.
\newblock URL \url{http://proceedings.mlr.press/v75/dwork18a.html}.

\bibitem[Dwork and Lei(2009)]{DworkLei}
C.~Dwork and J.~Lei.
\newblock Differential privacy and robust statistics.
\newblock In M.~Mitzenmacher, editor, \emph{STOC}, pages 371--380. ACM, 2009.

\bibitem[Dwork et~al.(2006{\natexlab{a}})Dwork, Kenthapadi, McSherry, Mironov,
  and Naor]{DKMMN06}
C.~Dwork, K.~Kenthapadi, F.~McSherry, I.~Mironov, and M.~Naor.
\newblock Our data, ourselves: Privacy via distributed noise generation.
\newblock In S.~Vaudenay, editor, \emph{EUROCRYPT}, volume 4004 of
  \emph{Lecture Notes in Computer Science}, pages 486--503. Springer,
  2006{\natexlab{a}}.

\bibitem[Dwork et~al.(2006{\natexlab{b}})Dwork, McSherry, Nissim, and
  Smith]{DMNS06}
C.~Dwork, F.~McSherry, K.~Nissim, and A.~Smith.
\newblock Calibrating noise to sensitivity in private data analysis.
\newblock In \emph{TCC}, volume 3876 of \emph{Lecture Notes in Computer
  Science}, pages 265--284. Springer, 2006{\natexlab{b}}.

\bibitem[Dwork et~al.(2010)Dwork, Rothblum, and Vadhan]{DRV10}
C.~Dwork, G.~N. Rothblum, and S.~P. Vadhan.
\newblock Boosting and differential privacy.
\newblock In \emph{FOCS}, pages 51--60. IEEE Computer Society, 2010.

\bibitem[Edelsbrunner(1987)]{Edelsbrunner87book}
H.~Edelsbrunner.
\newblock \emph{Algorithms in Combinatorial Geometry}.
\newblock Springer-Verlag, Berlin, Heidelberg, 1987.
\newblock ISBN 0-387-13722-X.

\bibitem[Feldman and Xiao(2015)]{FX14}
V.~Feldman and D.~Xiao.
\newblock Sample complexity bounds on differentially private learning via
  communication complexity.
\newblock \emph{{SIAM} J. Comput.}, 44\penalty0 (6):\penalty0 1740--1764, 2015.
\newblock \doi{10.1137/140991844}.
\newblock URL \url{http://dx.doi.org/10.1137/140991844}.

\bibitem[G{\"{a}}rtner and Welzl(1994)]{Gartner94vapnik}
B.~G{\"{a}}rtner and E.~Welzl.
\newblock Vapnik-{C}hervonenkis dimension and (pseudo-)hyperplane arrangements.
\newblock \emph{Discrete {\&} Computational Geometry}, 12:\penalty0 399--432,
  1994.
\newblock \doi{10.1007/BF02574389}.
\newblock URL \url{https://doi.org/10.1007/BF02574389}.

\bibitem[Hsu et~al.(2014)Hsu, Roth, Roughgarden, and Ullman]{HsuRRU14}
J.~Hsu, A.~Roth, T.~Roughgarden, and J.~Ullman.
\newblock Privately solving linear programs.
\newblock In \emph{Automata, Languages, and Programming - 41st International
  Colloquium, {ICALP} 2014, Proceedings, Part {I}}, pages 612--624, 2014.
\newblock \doi{10.1007/978-3-662-43948-7\_51}.
\newblock URL \url{https://doi.org/10.1007/978-3-662-43948-7\_51}.

\bibitem[Kasiviswanathan et~al.(2011)Kasiviswanathan, Lee, Nissim,
  Raskhodnikova, and Smith]{KLNRS11}
S.~P. Kasiviswanathan, H.~K. Lee, K.~Nissim, S.~Raskhodnikova, and A.~D. Smith.
\newblock What can we learn privately?
\newblock \emph{{SIAM} J. Comput.}, 40\penalty0 (3):\penalty0 793--826, 2011.
\newblock \doi{10.1137/090756090}.
\newblock URL \url{https://doi.org/10.1137/090756090}.

\bibitem[Kearns(1998)]{Kearns98}
M.~J. Kearns.
\newblock Efficient noise-tolerant learning from statistical queries.
\newblock \emph{J. ACM}, 45\penalty0 (6):\penalty0 983--1006, 1998.

\bibitem[{Liu} et~al.(2014){Liu}, {Mosler}, and {Mozharovskyi}]{Xiaohui14Tukey}
X.~{Liu}, K.~{Mosler}, and P.~{Mozharovskyi}.
\newblock {Fast computation of Tukey trimmed regions and median in dimension
  $p>2$}.
\newblock \emph{arXiv e-prints}, arXiv:1412.5122:\penalty0 arXiv:1412.5122,
  2014.

\bibitem[McSherry and Talwar(2007)]{MT07}
F.~McSherry and K.~Talwar.
\newblock Mechanism design via differential privacy.
\newblock In \emph{FOCS}, pages 94--103. IEEE Computer Society, 2007.

\bibitem[Miller and Sheehy(2010)]{Miller10approx}
G.~L. Miller and D.~R. Sheehy.
\newblock Approximate centerpoints with proofs.
\newblock \emph{Comput. Geom.}, 43\penalty0 (8):\penalty0 647--654, 2010.
\newblock \doi{10.1016/j.comgeo.2010.04.006}.
\newblock URL \url{https://doi.org/10.1016/j.comgeo.2010.04.006}.

\bibitem[Shalev-Shwartz and Ben-David(2014)]{Shalev14book}
S.~Shalev-Shwartz and S.~Ben-David.
\newblock \emph{Understanding Machine Learning: From Theory to Algorithms}.
\newblock Cambridge University Press, New York, NY, USA, 2014.
\newblock ISBN 1107057132, 9781107057135.

\bibitem[Tukey(1975)]{Tuk75}
J.~W. Tukey.
\newblock Mathematics and the picturing of data.
\newblock In \emph{Proc. Int. Congress of Mathematicians}, volume~2, pages
  523--532, 1975.

\bibitem[Valiant(1984)]{Valiant84}
L.~G. Valiant.
\newblock A theory of the learnable.
\newblock \emph{Commun. ACM}, 27\penalty0 (11):\penalty0 1134--1142, Nov. 1984.
\newblock ISSN 0001-0782.
\newblock \doi{10.1145/1968.1972}.
\newblock URL \url{http://doi.acm.org/10.1145/1968.1972}.

\bibitem[Yaglom and Boltyanski{\v \i}(1961)]{Yaglom61book}
I.~M. Yaglom and V.~G. Boltyanski{\v \i}.
\newblock \emph{Convex figures}.
\newblock Holt, Rinehart and Winston, 1961.

\end{thebibliography}

\end{document}